\documentclass[letterpaper, 10 pt, conference]{ieeeconf}  

\IEEEoverridecommandlockouts                              

\usepackage{tikz}
\usetikzlibrary{arrows.meta,calc,positioning, arrows, backgrounds, fit}

\usepackage{enumerate}
\usepackage{amsmath,amssymb,amsfonts}
\usepackage{mathtools}
\usepackage{graphicx}
\usepackage{subcaption}
\usepackage{bm}
\usepackage{algorithm}
\usepackage{booktabs}
\usepackage{algpseudocode}
\newtheorem{proposition}{Proposition}

\newtheorem{remark}{Remark}
\newtheorem{problem}{Problem}
\usepackage[nolist,nohyperlinks]{acronym}
\begin{acronym}
	\acro{bmc}[BMC]{Bounded Model Checking}
	\acro{odd}[ODD]{Operational Design Domain}
	\acro{acc}[ACC]{Adaptive Cruise Control}
	\acro{mpc}[MPC]{Model Predictive Control}
	\acro{lka}[LKA]{Lane Keeping Assist}
	\acro{qp}[QP]{Quadratic Programming}
 	\acro{dnn}[DNN]{Deep Neural Network}
    \acro{ai}[AI]{Artificial Intelligence}
    \acro{cis}[CIS]{Control Invariant Set}
    \acro{smt}[SMT]{Satisfiability Modulo Theory}
    \acro{cps}[CPS]{Cyber-Physical Systems}
    \acro{ltl}[LTL]{Linear Temporal Logic}
    \acro{ekf}[EKF]{Extended Kalman Filter}
    \acro{clf}[CLF]{Control Lyapunov Function}
    \acro{cbf}[CBF]{Control Barrier Function}
    \acro{lpv}[LPV]{Linear Parameter-Varying}
    \acro{pwa}[PWA]{Piece-Wise Affine}
    \acro{lti}[LTI]{Linear Time Invariant}
    \acro{hjb}[HJB]{Hamilton–Jacobi-Bellman}
    \acro{agc}[AGC]{Assume-Guarantee Contract}
    \acro{stl}[STL]{Signal Temporal Logic}
    \acro{ros}[ROS]{Robot Operating System}
    \acro{rv}[RV]{Runtime Verification}
    \acro{rpi}[RCIS]{Robust Controlled Invariant Set}
    \acro{adas}[ADAS]{Advanced Driver Assistance System}
    \acro{ddpg}[DDPG]{Deep Deterministic Policy Gradient}
    \acro{method}[CoCoSaFe]{\underline{Co}mpositional \underline{Co}de-level \underline{Sa}fety verification \underline{F}ram\underline{e}work}
\end{acronym}
\newcommand{\Obs}{\mathrm{Obs}} 

\title{\LARGE \bf
Monitoring autonomous persistent surveillance missions using invariance
}

\author{Vladislav Nenchev and Prodromos Sotiriadis
\thanks{V. Nenchev and P. Sotiriadis are with Department of Electrical and Computer Engineering, University of the Bundeswehr Munich, 85579 Neubiberg, Germany.
        {\{\tt\small vladislav.nenchev, prodromos.sotiriadis\}@unibw.de}}%
}

\begin{document}

\maketitle
\thispagestyle{empty}
\pagestyle{empty}

\begin{abstract}
This paper studies runtime monitoring for persistent surveillance by autonomous robots when the autonomy stack is a black box. The environment is partitioned into finitely many parts, each carrying an uncertainty state that decreases when observed and increases otherwise. We model the closed loop as a state-dependent hybrid system with linear parameter varying dynamics and design a monitor based on an invariant computed offline. As this invariant is typically hard to obtain for large to-be-surveyed spaces, we propose a compositional monitor obtained by decentralized computation of low-dimensional invariant sets for each uncertainty region, and checking their conjunction online. Under common independence assumptions, the compositional monitor is sound and complete with respect to the full-system invariant. The approach is applied in a case study with a real robot persistently monitoring a labyrinth, emphasizing its applicability in practice. 
\end{abstract}

\section{Introduction}
In persistent surveillance, an agent moves through an environment to collect information about specific targets over time. It is applicable across a wide range of robotic problems, such as ocean monitoring \cite{smith_persistent_2011} and forest fire detection \cite{casbeer_cooperative_2006}. In many deployed autonomous systems, navigation stacks mix \ac{ai}-based modules with proprietary or otherwise black-box components. Unanticipated disturbances, sensor faults, and software vulnerabilities can trigger navigation stalls, collisions, or critical failures. In addition, cyber attacks that spoof sensors or manipulate control channels can deliberately degrade performance or subvert mission objectives. This motivates a runtime monitoring layer that reasons from measured state (and, when available, provided inputs) and raises early alerts before loss of progress manifests operationally.

This paper studies monitoring the operation of a robot performing persistent surveillance when the surveyed workspace is partitioned into a finite number of parts. We model the uncertainty state of each part as a \ac{lpv} system, as we assume that the robot has the ability to take (noisy) measurements of the state; this is a typical setting that can be found in many persistent surveillance approaches such as \cite{lan_planning_2013,chen_deep_2020,ha_periodic_2019,hall_persistent_2025}.
The uncertainty \emph{decreases} while remaining within the part and \emph{increases} otherwise.
We represent the closed loop as a state-dependent hybrid system characterized by \ac{pwa} dynamics, incorporating bounded, region-dependent \ac{lpv} uncertainty, and compute the \ac{rpi} \cite{blanchini_set_1999}. Intuitively, if the current measurement lies in a \ac{rpi}, then for all admissible inputs and parameter variations consistent with the current region, the next state also remains inside. This property of the \ac{rpi} is used for runtime monitoring.

Computing the maximal \ac{rpi} in the \emph{full} space that contains the robot state, all uncertainties of the to-be-surveyed regions, and inputs becomes intractable as the number of parts of the environment grows. We propose a \emph{compositional} construction: for each part we compute the \ac{rpi} in robot pose, input, and part-uncertainty only, and the online monitor checks the \emph{conjunction} of these per-part memberships. Assuming that the uncertainties of the regions are independent, we show that the intersection of per-part \ac{rpi} is \emph{equivalent} to the full-system \ac{rpi}. Thus, the compositional monitor is both \emph{sound} and \emph{complete}. Its complexity scales linearly with the number of parts. We demonstrate the effective application of the proposed monitor in a case study with a four-wheeled differential-drive robot persistently surveying a labyrinth. 

The contributions of the paper are as follows:
\begin{enumerate}[i.] 
  \item Formalize persistent surveillance with black-box control as a hybrid system whose uncertainties evolve via bounded \ac{lpv} factors.
  \item Compute per-part \ac{rpi} and prove their \emph{soundness} and \emph{completeness} with respect to the full \ac{rpi}.
  \item Propose a monitor based on the compositional \ac{rpi} with polyhedral modes and constant-time online membership checks suitable for embedded deployment.
  \item Demonstrate the effectiveness of the approach to monitor persistent surveillance of a labyrinth by a real robot.
\end{enumerate}

The remainder of the paper is structured as follows. Section~II discusses related work. Section~III states the monitoring problem for persistent surveillance and introduces a running example used throughout the paper. Section~IV presents the monitoring approach based on formalizing the hybrid model, obtaining invariants, and showing how to use them for monitoring. Section~V evaluates the approach in a laboratory case study. Section~VI discusses assumptions, limitations, and extensions; Section~VII concludes the paper.

\section{Related Work}

A large body of work focuses on patrols over partitioned environments, ranging from waypoint/visitation sequencing \cite{yu_persistent_2015,jones_information-guided_2015} to hybrid planners that account for region-dependent dynamics \cite{hall_persistent_2025}, as well as multi-robot coordination \cite{zhou_optimal_2018}. These approaches output trajectories or visitation schedules and typically provide completeness or optimality guarantees under the modeled dynamics and costs, assuming access to (or direct control over) the planning and control interfaces. Necessary and sufficient conditions on robot speed controllers along a predefined path were derived for keeping the accumulation function for persistent surveillance bounded \cite{smith_persistent_2011-1}.
Prior work has also addressed autonomy under formal specifications \cite{nenchev_receding_2016} and an event-triggered task structure \cite{nenchev_event-driven_2018}, in which policies are computed from modeled dynamics under uncertainty. Our runtime monitoring is complementary to all of these approaches: a planner can be wrapped by our monitor to guard against stalls, sensor faults, or adversarial perturbations that invalidate planning assumptions.

\ac{rv} frameworks monitor executions against formal specifications with strong correctness guarantees \cite{leucker_brief_2009}. Temporal logic enables quantitative monitoring of real-valued traces \cite{maler_monitoring_2004,fainekos_robustness_2009}. Although \ac{rv} offers rigorous semantics and mature tooling, monitoring rich specifications can be computationally heavy onboard resource-constrained robots, and many approaches assume gray-/white-box access to the system. Progress and failure monitors based on stability certificates and short-horizon rollouts \cite{nenchev_monitoring_2026} were proposed for goal-directed navigation, which, however, are not easily extendable for persistent surveillance.

Hamilton-Jacobi reachability can also be used to provide monitors with safety guarantees via value functions over the state space \cite{mitchell_time-dependent_2005}. The resulting certificates reduce runtime checks to set membership, but memory and computation requirements suffer from the curse of dimensionality, thus, limiting scalability for robotic systems. A \ac{cbf} may enforce forward invariance via solving a \ac{qp} online, offering safety guarantees with minimal intervention \cite{ames_control_2017}. Model predictive control yields constraint satisfaction and robustness through tube/robust formulations \cite{mayne_model_2014}. They require either solving optimization problems at runtime or access to the control loop, which may be incompatible with black-box autonomy.

Compared to logic- or optimization-heavy monitors, our method relies on precomputed \ac{rpi} sets and set-membership checks online, while retaining a provably sound and complete compositional construction that scales to larger environments.

\section{Problem Statement}

Let a robot with pose $x_k \in X$ at discrete time $k\in\mathbb{N}$ and a bounded velocity $u_k \in U\subset \mathbb{R}^2$ move according to
\begin{equation}\label{eq:motion}
x_{k+1} = x_k + u_k T_s,
\end{equation}
where $T_s$ is the sampling time. The autonomy stack that generates $u_k$ is treated as a \emph{black box}.

For the persistent surveillance setup, consider a bounded planar workspace $X \subset \mathbb{R}^2$ partitioned into finitely many parts $P=\{P_1,\ldots,P_N\}, P_i\subset X$ with pairwise disjoint interiors; parts may meet only on their boundaries (e.g., along shared edges or vertices). For each part $P_i$, let $\Obs_i\subset X$ denote the set of robot poses from which $P_i$ is observed. Note that $\Obs_i$ may also consist of multiple (disconnected) parts. The uncertainty of part $P_i$ is denoted by a scalar $z_{i,k}\in\mathbb{R}_{+}$ at time $k$. Uncertainty \emph{decreases} while $P_i$ is effectively observed, i.e., when $x\in\Obs_i$, and \emph{increases} otherwise. For clarity of exposition, we assume an obstacle-free workspace, but we show how the presence of obstacles can be treated in the case study. 
\paragraph*{Assumptions}
(i) observability sets for each part are known or conservatively over-approximated; (ii) each $z_{i,k}$ evolves independently of $z_{j\neq i}$ given $x_k$ (no cross-coupling among parts beyond the shared pose); (iii) uncertainty updates are bounded and monotone with observation status, e.g., multiplicative LPV factors with decay when observed and growth when unobserved; (iv) disturbances/parametric variations (including benign faults or adversarial perturbations within specification) are captured by these bounds; (v) at runtime, the monitor receives $x_k$; it may also receive $u_k$ (optional) and either measured $z_{i,k}$ or computable proxies derived from sensing/coverage logs. 
\begin{problem}
Given the available signals, design a \emph{runtime monitor} that determines whether the persistent surveillance is being met, i.e., all parts remain within acceptable uncertainty, $z_{i,k}\in[0,z_i^{\mathrm{crit}}]$ for desired thresholds $z_i^{\mathrm{crit}}>0$ , and that raises (anticipatory) \emph{alerts} when a violation is imminent.
\end{problem}
The monitor should run at predictable, low per-step cost suitable for onboard deployment, provide per-part diagnostics (which regions are at risk and by how much), scale to large $N$, and remain robust to bounded disturbances, parameter drift, and adversarial perturbations that are consistent with the specified observation model.

\begin{figure}[t]
\centering
\begin{tikzpicture}[scale=0.8, >=LaTeX]

\def\Xmax{5}
\def\Ymax{3.5}
\def\S{0.7}            
\coordinate (P1) at (1,1.5);
\coordinate (P2) at (4.,1.5);
\coordinate (X0) at (2.5,0.5); 

\tikzset{
  ws/.style={draw=black, thick},
  gridl/.style={step=1cm, very thin, color=black!10},
  footprint/.style={fill=blue!8, draw=blue!60, thick},
  traj/.style={draw=black, line width=1.1pt, -{Latex[length=3mm]}},
  point/.style={circle, draw=black, fill=black, inner sep=1.4pt},
  label/.style={font=\scriptsize, fill=white, inner sep=1pt},
    body/.style={line width=0.8pt, draw=black, fill=white},
  wheelfill/.style={fill=black!80, draw=black!80}
}

\draw[gridl] (0,0) grid (\Xmax,\Ymax);
\draw[ws] (0,0) rectangle (\Xmax,\Ymax);

\foreach \x in {0,1,...,5} {
  \draw (\x,0) -- ++(0,0.08);
  \node[below, font=\scriptsize] at (\x,0) {\x};
}
\foreach \y in {0,1,...,3} {
  \draw (0,\y) -- ++(0.08,0);
  \node[left, font=\scriptsize] at (0,\y) {\y};
}
\node[below right] at (2.25,-0.2) {$x$};
\node[below left]  at (-0.2,0.55*\Ymax) {$y$};

\draw[footprint] ($(P1)+(-\S,-\S)$) rectangle ($(P1)+(\S,\S)$);
\draw[footprint] ($(P2)+(-\S,-\S)$) rectangle ($(P2)+(\S,\S)$);

\node[point] at (P1) {};
\node[point] at (P2) {};
\node[above left=1pt of P1, label] {$P_1$};
\node[above right=1pt of P2,  label] {$P_2$};

\draw[traj]
  (X0) -- ++(0.0,0.9) -- ++(-0.3,0.7) -- (P1)
  -- ++(0.9,0.0) -- ++(1.0,0.0) -- ++(1.1,0.0) -- (P2)
  -- ++(-0.8,0.4) -- ++(-1.0,-0.2) -- ++(-1.2,-0.5) -- ($(P1)+(0.3,-0.4)$);

\node[point] at (X0) {};
\begin{scope}[shift={(X0)}]
  \def\bl{0.28}   
  \def\bw{0.18}   
  \def\whl{0.22}  
  \def\wth{0.03}  
  \def\cr{0.04}   

  \draw[body] (-\bw,-\bl) rectangle (\bw,\bl);

  \draw[wheelfill] (-\bw-2*\wth, -\whl) rectangle (-\bw, \whl);
  \draw[wheelfill] ( \bw,        -\whl) rectangle ( \bw+2*\wth, \whl);

  \draw[wheelfill] (0, \bl+0.06) circle (\cr);

  \draw[->, line width=0.8pt] (0,0) -- (0,\bl+0.18);
\end{scope}
\node[below right=1pt of X0, label] {start};

\begin{scope}[shift={(0.3,3.2)}]
  \draw[footprint] (-0.2,-0.2) rectangle +(0.4,0.4);
  \node[right=6pt] at (0.2,0) {\scriptsize square sensing footprint};
  \draw[traj] (-0.2,-0.55) -- +(0.4,0.0);
  \node[right=6pt] at (0.2,-0.55) {\scriptsize example trajectory};
\end{scope}

\end{tikzpicture}
\caption{Example: A robot persistently surveying two parts $P_1,P_2$ with a sample patrol trajectory.}
\label{fig:running-example}
\end{figure}
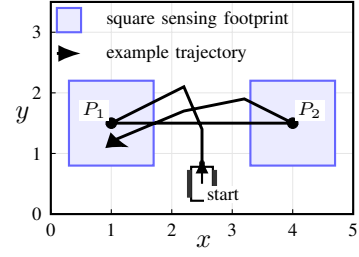

\emph{Example.}
We instantiate $X = [0,5]\times[0,3.5]$ with two parts centered at $p_1=(1,1.5)$ and $p_2=(4,1.5)$. Each part has a \emph{square sensing footprint} of half-side $r=0.7$, i.e., $\mathsf{obs}_i(x)=1 \;\;\Longleftrightarrow\;\; \|x-p_i\|_{\infty} \le r ,$ so the observation regions are polyhedral and used \emph{as is} in computation. The sampling time is $T_s=1$\,s and the robot follows the single-integrator model $x_{k+1}=x_k+u_k T_s$ with input bound $\|u_k\|_{\infty}\le 0.6$ (i.e., $U=\{u:\|u\|_{\infty}\le 0.6\}$). Uncertainties have to satisfy $0 \le z_i \le z^{\text{crit}}=4$. Observed/unobserved multiplicative factors lie in intervals $a \in [0.65,0.80]\,$ and $b \in [1.05,1.15]\,$; for robustness we use the monotone worst-case vertex $(a_{\max},b_{\max})=(0.80,1.15)$ in each region. This example shown in Figure~\ref{fig:running-example} will illustrate the proposed solution throughout the paper.

\section{Invariant-Based Monitoring}
\begin{figure}
  \centering
  \begin{tikzpicture}[->,>=stealth',auto,semithick,    arrow/.style={->, thick}]
    \node[draw, rectangle, rounded corners, fill=gray!10, minimum width=8.5cm, minimum height=1.5cm] (offline) at (-1.5,2) {};
     \node[anchor=north west] at (offline.north west) {\textbf{Offline}};
    \node[draw, rectangle, rounded corners, fill=gray!10, minimum width=8.5cm, minimum height=1.5cm] (online) at (-1.5,0.3) {};
     \node[anchor=north west] at (online.north west) {\textbf{Online}};
    \node[draw, rectangle, rounded corners, minimum width=2cm, fill=blue!20, align=left] (model) at (-4,1.85) {Uncertainty \&\\ robot model};
        \node[draw, rectangle, rounded corners, minimum width=2cm, fill=blue!20, align=left] (op) at (-1.5,1.85) {Operational\\ limits};
    \node[draw, rectangle, rounded corners, minimum width=2cm, fill=blue!20, align=left] (lyap) at (1,1.85) {Invariant\\set};
    \node[draw, rectangle, rounded corners, minimum width=2cm, fill=green!30, align=left] (ekf) at (-1.5,0.3) {Measured\\ state (\& input)};
    \node[draw, rectangle, rounded corners, minimum width=2cm, fill=green!30, align=left] (monitor) at (1,0.3) {Runtime\\ monitoring};
      \draw[arrow] (model) -- (op);
            \draw[arrow] (op) -- (lyap);
    \path (lyap) edge[out=270,in=90] (monitor)
          (ekf) edge (monitor);
  \end{tikzpicture}
  \caption{Overview of the monitoring approach: In the \textit{Offline} stage (top) the certificate is synthesized from the robot motion model and operational limits; in the \textit{Online} stage (bottom) the measured state (and, optionally, the input) are used to check if operation remains in the invariant.}
  \label{fig:overview}
\end{figure}
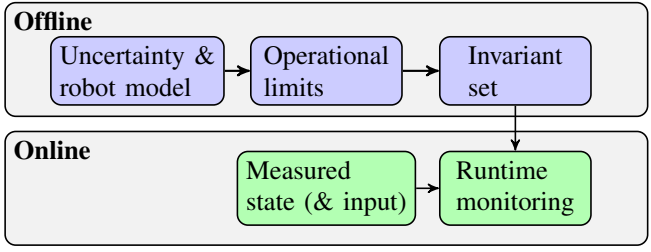
This section presents the key components of the approach as shown in Fig.\,\ref{fig:overview}: formalizing the hybrid model for persistent surveillance, introducing the \ac{rpi} and a compositional construction with soundness/completeness guarantees, computing them, and using invariants as monitoring primitives.

\subsection{Hybrid model}
We model the closed loop as a state-dependent hybrid system that contains the motion of the robot \eqref{eq:motion} and multiplicative, region-dependent uncertainty dynamics:
\begin{align}
z_{i,k+1} &= \gamma_i(x_k)\, z_{i,k}, \qquad i=1,\dots,N, \label{eq:unc}
\end{align}
where $\gamma_i(x)$ is changes according to the robot’s location:
\[
\gamma_i(x) \in 
\begin{cases}
[a_i^{\min},a_i^{\max}] \subset (0,1), & x \in \mathrm{Obs}_i,\\[2pt]
[b_i^{\min},b_i^{\max}] \subset (1,\infty), & \text{else}.
\end{cases}
\]
For each part $i$, we build a PWA system in the state $\xi_i=(x,z_i,u)$ with two modes:
\[
\begin{aligned}
\text{obs}_i: & x^+=x+uT_s,\quad z_i^+=a_i^{\max} z_i,\quad x\in\mathrm{Obs}_i;\\
\text{unobs}_i: & x^+=x+uT_s,\quad z_i^+=b_i^{\max} z_i,\quad x\in X\setminus\mathrm{Obs}_i,
\end{aligned}
\]
subject to the joint domain $D_i:=X\times[0,z_i^{\mathrm{crit}}]\times U$. Further, define guards $G_i^{\mathrm{obs}} := \{x\in X \mid x\in\Obs_i\}$ and
$G_i^{\mathrm{ubobs}} := \{x\in X \mid x\notin\Obs_i\}$.

\emph{Example (revisited).}
In our running example, for each $i\in\{1,2\}$ we build a PWA subsystem in $[x^\top,\,z_i]^\top\in\mathbb{R}^3$ with input $u\in\mathbb{R}^2$ and two mode maps:
\[
\text{obs}_i:\;\; \begin{cases} x^+=x+uT_s, \\ z_i^+=a_{\max}\,z_i, \end{cases}
\qquad
\text{unobs}_i:\;\; \begin{cases} x^+=x+uT_s, \\ z_i^+=b_{\max}\,z_i, \end{cases}
\]
guarded by $x\in\mathrm{Obs}_i:=\{\,\|x-p_i\|_\infty\le r\,\}$ and $x\in X\setminus\mathrm{Obs}_i$, respectively. Because $X\setminus\mathrm{Obs}_i$ is non-convex, it is represented by a union of polyhedral sets with a separate linear-time-invariant-mode per set. To reduce the complexity of the example, we consider only the mode between the two to-be-observed parts, and the overall hybrid system consists of three modes. 

\subsection{Invariants}
Let $\xi_k := [\,x_k^\top,\ [z_{1,k},\dots,z_{N,k}],\ u_k^\top\,]^\top$ denote the joint state–input vector with domain
\[
D := X\times\prod_{i=1}^N[0,z_i^{\mathrm{crit}}]\times U.
\]
A set $\mathcal{S}$ is \ac{rpi} for \eqref{eq:motion}–\eqref{eq:unc} if, for all admissible $u_k\in U$ and for all factors $\gamma_i(\cdot)$ consistent with the active region at $x_k$, $\xi_k \in \mathcal{S} \;\Rightarrow\; \xi_{k+1} \in \mathcal{S}$.

Computing the maximal full-system \ac{rpi} $\mathcal{S}$ scales poorly with the number of parts $N$, making exact computation and storage quickly prohibitive. Thus, we compute per-part \ac{rpi}'s and intersect them to retain correctness. For each part $i$, consider the subsystem over the state space $(x,z_i)$, dynamics \eqref{eq:motion} with the input set $U$ and \eqref{eq:unc} only for $z_i$, and the corresponding region guards. Then, compute its maximal \ac{rpi} $\mathcal{S}_i^{\max}$ in $D_i := X\times[0,z_i^{\mathrm{crit}}]\times U$. For multiplicative updates $z_i^+=\gamma_i z_i$ with $\gamma_i$ bounded per region, the next-step map is monotone in $\gamma_i$. Hence, the \ac{rpi} set coincides with the set computed at the region-wise \emph{worst-case} vertices $(a_i^{\max}, b_i^{\max})$ (equivalently, intersect vertex-wise invariants if multiple vertices are used). For a set $S_i\subseteq D_i$, define the one–step predecessors in mode $i$ as
\[
\mathrm{Pre}_i(S_i){:=}\bigl\{\,\hat{x}\in D_i:\ f_i(\hat{x})\in S_i,\ \forall u\in U,\ \forall\,\gamma_i\in\Gamma_i(x)\,\bigr\}.
\]
This leads to the following proposition.

\begin{proposition}\label{prop:mod-eq-full}
Assume:
\begin{enumerate}[(i)] 
\item $\forall i$, $z_i^+$ in~\eqref{eq:unc} depends only on $z_i$ and the guard $Obs_i$, and not on $z_j$ with $j\neq i$;
\item The full system and each per-part subsystem use the same workspace set $X$ and input set $U$;
\item The full system and each per-part subsystem use the same guards $G^{\text{obs}}_i$, $G^{\text{unobs}}_i$ and uncertainty/LPV bounds.
\end{enumerate}
Let $\mathcal{S}\subseteq D$ denote the maximal \ac{rpi} set for the full system, and let $\mathcal{S}_i^{\max}\subseteq D_i$ denote the maximal \ac{rpi} set for the $i$th per-part subsystem. Then, with $i=1,\dots,N$ and 
\[
\hat{\mathcal{S}} \;:=\; \big\{[x^\top,z^\top,u^\top]^\top\in D \,\big|\, [x^\top,z_i,u^\top]^\top\in \mathcal{S}_i^{\max}\ \ \forall i\big\},
\]
$\mathcal{S}=\hat{\mathcal{S}}$ holds.
\end{proposition}
\begin{proof}
We show both inclusions.

\medskip\noindent\emph{1) $\mathcal{S}\subseteq \hat{\mathcal{S}}$.}
Fix $i\in\{1,\dots,N\}$ and consider the projection
\[
\Pi_i: D\to D_i,\qquad \Pi_i(x,z,u) := (x,z_i,u).
\]
Let $P_i := \Pi_i(\mathcal{S})\subseteq D_i$. We claim that $P_i$ is \ac{rpi} for the $i$th subsystem.
Take any $(\bar x,\bar z_i,\bar u)\in P_i$. By definition, there exists $\bar z_{-i}$ such that $(\bar x,\bar z,\bar u)\in \mathcal{S}$.
Since $\mathcal{S}$ is \ac{rpi} for the full system, for any admissible uncertainty realization (equivalently, any $\gamma(\cdot)\in\Gamma(\bar x)$) and the induced successor
\[
\bar x^+ = \bar x + \bar uT_s,\qquad \bar z^+ = g(\bar x,\bar z),
\]
we have $(\bar x^+,\bar z^+, \tilde u)\in\mathcal{S}$ for all $\tilde u\in U$ (as required by the predecessor definition used in the paper).
By Assumption (i), $\bar z_i^+$ depends only on $(\bar x,\bar z_i)$ (and the guard $Obs_i$), hence $\bar z_i^+ = g_i(\bar x,\bar z_i)$ is exactly the successor of the $i$th subsystem under the same guard and uncertainty bounds (Assumption (iii)).
Projecting yields $(\bar x^+,\bar z_i^+,\tilde u)\in P_i$ for all $\tilde u\in U$ and all admissible uncertainties, i.e., $P_i$ is \ac{rpi} for the $i$th subsystem.
By maximality of $\mathcal{S}_i^{\max}$ on $D_i$, we obtain $P_i\subseteq \mathcal{S}_i^{\max}$, hence
\[
(x,z,u)\in\mathcal{S}\ \Rightarrow\ (x,z_i,u)\in \mathcal{S}_i^{\max}\ \ \forall i
\ \Rightarrow\ (x,z,u)\in \hat{\mathcal{S}}.
\]
Therefore, $\mathcal{S}\subseteq \hat{\mathcal{S}}$.

\medskip\noindent\emph{2) $\hat{\mathcal{S}}\subseteq \mathcal{S}$.}
We first show that $\hat{\mathcal{S}}$ is \ac{rpi} for the full system.
Take any $(\bar x,\bar z,\bar u)\in \hat{\mathcal{S}}$. By definition, $(\bar x,\bar z_i,\bar u)\in \mathcal{S}_i^{\max}$ for all $i$.
Let an admissible uncertainty realization be fixed (equivalently, choose any $\gamma(\cdot)\in\Gamma(\bar x)$), with successor
\[
\bar x^+ = \bar x + \bar uT_s,\qquad \bar z_i^+ = g_i(\bar x,\bar z_i),\ \ i=1,\dots,N.
\]
Because each $\mathcal{S}_i^{\max}$ is \ac{rpi} for the corresponding per-part subsystem and, by Assumptions (ii)–(iii), uses the same $X,U$, guards, and uncertainty bounds as the full system,
\[
(\bar x^+,\bar z_i^+,\tilde u)\in \mathcal{S}_i^{\max}\qquad \forall i,\ \forall \tilde u\in U.
\]
Thus, $(\bar x^+,\bar z^+,\tilde u)\in \hat{\mathcal{S}}$ for all $\tilde u\in U$ and all admissible uncertainties, i.e., $\hat{\mathcal{S}}$ is \ac{rpi} for the full system and satisfies $\hat{\mathcal{S}}\subseteq D$.
By maximality of $\mathcal{S}$ among \ac{rpi} subsets of $D$, we conclude $\hat{\mathcal{S}}\subseteq \mathcal{S}$.

Combining 1) and 2) gives $\mathcal{S}=\hat{\mathcal{S}}$.
\end{proof}
\begin{remark}
Proposition~\ref{prop:mod-eq-full} implies that the compositional invariant is \emph{sound} (no false alarms) and \emph{complete} (no misses) w.r.t.\ $\mathcal{S}\subseteq D$. Under static, known obstacles, the obstacle‑free portion of $X$ can be under‑approximated by a finite collection of convex polyhedral cells; this under‑approximation preserves soundness (see case study). Dynamic obstacles are not considered in this work.
\end{remark}

\subsection{Computing invariant sets}

The \ac{rpi} sets are obtained as \emph{greatest fixed points} of the monotone set operator
$T(S)=D\cap\text{Pre}(S)$, iterated from $S^{(0)}=D$ until convergence \cite{blanchini_set_1999}.
For linear (or modewise-affine) dynamics with polyhedral constraints, $\text{Pre}(\cdot)$ reduces to polyhedral operations (linear images, Minkowski sums, Pontryagin differences) and Boolean combinations, yielding polyhedral \ac{rpi} sets (exact or convergent approximations).
For the LPV states \eqref{eq:unc}, we use the standard \emph{vertex method}: compute (or intersect) invariants for each
vertex system $(a_i^{\max},b_i^{\max})$ consistent with the active region, which is conservative and monotone
with respect to multiplicative uncertainty \cite{blanchini_set_1999,shamma_overview_2012}.
We encode \eqref{eq:motion}–\eqref{eq:unc} as a hybrid LPV/PWA model with polyhedral constraints and compute the \ac{rpi}'s via fixed-point iteration.
To obtain the compositional invariants, we iterate the robust predecessor operator per mode under its guard and intersect across modes until a fixed point is reached; the result is a union of polytopes $\mathcal{S}_i^{\max}\subset D_i$ (Algorithm~\ref{alg:inv}).

\begin{algorithm}[t]
\caption{Per-part compositional invariant (offline)}\label{alg:inv}
\begin{algorithmic}[1]
\Require $X$, $U$, centers $p_i$, half-side $S$, bounds $z_i^{\mathrm{crit}}$, worst-case $(a_i^{\max},b_i^{\max})$, $T_s$, tolerance $\varepsilon$, max iters $L_{\max}$
\State Define guards: $\mathrm{Obs}_i=\{x:\|x-p_i\|_\infty\le S\}$, $\mathrm{Unobs}_i=X\setminus\mathrm{Obs}_i$
\State Define mode maps on $(x,z_i,u)$:
\Statex \hspace{\algorithmicindent} Obs: $x^+=x+uT_s$, $z_i^+=a_i^{\max}z_i$; \quad Unobs: $x^+=x+uT_s$, $z_i^+=b_i^{\max}z_i$
\State Initialize $S^{(0)}:=D_i:=X\times[0,z_i^{\mathrm{crit}}]\times U$
\For{$\ell=0,1,\dots,L_{\max}$}
  \State $\widehat S^{(\ell+1)} := \text{Pre}_{\text{Obs},i}\bigl(S^{(\ell)}\bigr)\ \cap\ \text{Pre}_{\text{Unobs},i}\bigl(S^{(\ell)}\bigr)$
  \State $S^{(\ell+1)} := D_i \cap \widehat S^{(\ell+1)}$ \Comment{intersect with joint box}
  \State \textbf{if} $S^{(\ell+1)}=S^{(\ell)}$ (or $\mathrm{dist}_H<\varepsilon$) \textbf{then} \textbf{break}
  \State \textbf{end if}
\EndFor
\State $\mathcal{S}_i^{\max} \leftarrow$ minimal union-of-polytopes representation of $S^{(\ell+1)}$ (remove redundant facets/sets)
\State \textbf{return} $\mathcal{S}_i^{\max}$
\end{algorithmic}
\end{algorithm}

\emph{Example (revisited).}
We compute two per-part invariants in $[x,y,\ z_i,\ v_x,v_y]$ (two modes per part: observed/unobserved) and, for comparison, a single full invariant in $[x,y,\ z_1,z_2,\ v_x,v_y]$ (three modes: observing $P_1$, observing $P_2$, observing none). We set $(a_{\max},b_{\max})=(0.80,1.15)$ in the observed/unobserved regions. Because the parts are denoted by squares, there is no geometric approximation gap. Table~\ref{tab:complexity} compares the complexity of the invariants in terms of state dimension, number of system modes, invariant polytopes and vertices, and computation time on a general purpose laptop.  

\begin{table}[t]
\centering
\caption{Running example: invariant complexity.}
\label{tab:complexity}
\setlength{\tabcolsep}{6pt}
\begin{tabular}{lcccccc}
\toprule
System & State Dim. & \# modes & \# P & \# Vertices & Comp.\,time\,[s] \\
\midrule
Part 1 & 3 & 2 & 12 & 72 & 2.4 \\
Part 2 & 3 & 2 & 18&  108 & 3.4 \\
Full   & 4 & 3 &  610& 9280 & 3319 \\
\bottomrule
\end{tabular}
\end{table}

\subsection{Monitoring using invariant sets}
\paragraph*{Using full invariant}
At time $k$, the system \eqref{eq:motion}–\eqref{eq:unc} is healthy iff $\xi_k\in\mathcal{S}$. Leaving $\mathcal{S}$ flags a failure or, if $u_k$ is also available and the resulting $\xi_{k+1}\not\in \mathcal{S}$ based on \eqref{eq:motion} -- an imminent failure.

\paragraph*{Using compositional invariants}
At time $k$, form $\xi_{i,k}=[x_k^\top,\,z_{i,k},\,u_k^\top]^\top$ (or drop $u_k$ and encode its bound in the offline construction) and test membership in $\mathcal{S}_i^{\max}$ for \emph{all} $i$. Following Proposition~1, the compositional decision is
\[
\text{healthy at }k \iff (x_k,z_{i,k},u_k)\in \mathcal{S}_i^{\max}\ \ \forall i\quad(\text{i.e., } \xi_k\in \hat{\mathcal{S}}).
\]
Membership reduces to halfspace checks against the facets of the polytopes in the union; for the compositional monitor the per–step cost scales linearly with $N$.

\begin{algorithm}[t]
\caption{Compositional monitoring (online)}\label{alg:online}
\begin{algorithmic}[1]
\Require $\{\mathcal{S}_i^{\max}\}_{i=1}^N$, measurements $(x_k,\{z_{i,k}\})$, optional $u_k$
\For{$i=1$ to $N$}
  \State $h_i \leftarrow \mathbf{1}\bigl\{(x_k,z_{i,k},u_k)\in \mathcal{S}_i^{\max}\bigr\}$ \Comment{polyhedral union membership}
\EndFor
\State \textbf{if} $\prod_{i=1}^N h_i = 1$ \textbf{then} \textbf{return} Healthy
\State \textbf{else} \textbf{return} Alert (report indices $\{i: h_i=0\}$ and nearest-facet margins)
\end{algorithmic}
\end{algorithm}

\emph{Example (revisited).}
For a measured triple $(x_k,z_{1,k},z_{2,k})$ (and optional $u_k$) the system is declared \emph{healthy} iff
\[
(x_k,z_{1,k},u_k)\in \mathcal{S}_1^{\max}\quad\text{and}\quad (x_k,z_{2,k},u_k)\in \mathcal{S}_2^{\max}.
\]
The full invariant $\mathcal{S}$ can also be used for monitoring analogously. On our platform on average, the compositional check ran in 0.03\,ms per step, while the check with the full invariant took 1\,ms per step without any numerical optimizations.

\section{Case Study}
We evaluate the proposed invariant-based monitor on a four-wheeled differential-drive robot tasked with persistently surveying a labyrinth-like indoor environment. The study follows the workflow (Figure~\ref{fig:overview}) introduced earlier: per-part robust invariants are synthesized offline from a LPV-hybrid model and operational limits; the monitor performs set-membership checks to raise anticipatory alerts. We only evaluate the \emph{compositional} invariant monitor, as computing the \emph{full} invariant is intractable.

\subsection{Experimental Setup}
\paragraph*{Robot}
The platform is a ground robot (Figure~\ref{fig:mon:obs1}) with footprint $0.44\times 0.40$\,m (mass 5\,kg), wheel encoders (16 pulses/rev), an ITG-3200 gyroscope (10\,Hz), and a Hokuyo URG-04LX 2D LiDAR (10\,Hz, 240$^\circ$ FOV). Low-level control runs on an STM32F746ZG (Cortex-M7, 1\,kHz), while high-level software runs on a Raspberry Pi 4B (8\,GB) with ROS~2; actuator bounds are $v\in[-0.26,0.26]$\,m/s, $\omega\in[-1,1]$\,rad/s and the state space is $x,y\in[-4,4]\times[-3,3]$\,m, $\theta\in[-\pi,\pi]$, with sampling time $T_s=0.1$\,s. 
\begin{figure}
  \centering
    \begin{subfigure}{0.35\textwidth}
    \includegraphics[width=\textwidth]{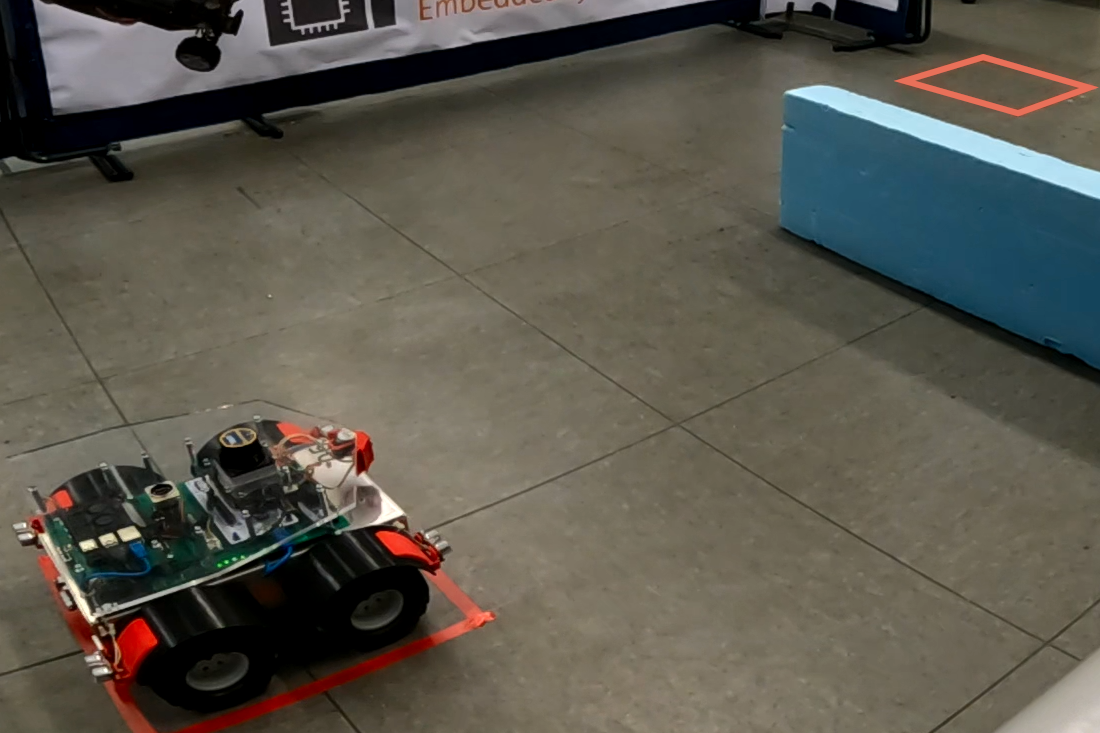}
    \caption{Robot navigation.}
    \label{fig:mon:obs1}
  \end{subfigure}
    \hfill
    \begin{subfigure}{0.48\textwidth}
    \includegraphics[width=\textwidth]{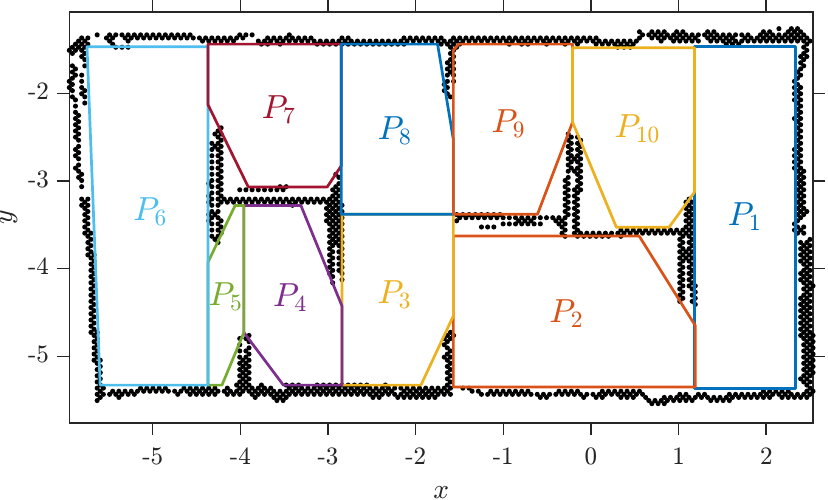}
    \caption{Labyrinth configuration.}
    \label{fig:mon:free1}
  \end{subfigure}
  \vspace{1ex}
  \caption{Labyrinth for persistent surveillance by a four‐wheeled differential‐drive robot. The labyrinth has been manually partitioned into 10 parts.}
  \label{fig:monitor}
\end{figure}

\paragraph*{Policy} We use a simple uncertainty–greedy patrol policy: the robot selects the partition with the largest current uncertainty $z_i$, plans a path to the center of that partition, and dwells until $z_i$ drops below a target level $z_i^{\mathrm{tar}}$; it then reselects the next part. A small hysteresis band on $z_i$ and a minimum inter-visit time $\tau_{\min}$ prevent chattering. Local motion is handled by ROS Nav2.

\paragraph*{Environment}
The traversable obstacle-free labyrinth workspace is manually under-approximated by a finite number of partially adjacent convex polyhedral cells that serve as observability regions and mode domains in the hybrid model. The resulting map and labels are shown in Figure~\ref{fig:mon:free1}. Each partition $P_i$ is equipped with an uncertainty proxy $z_i$ that evolves multiplicatively according to the LPV model with part–specific factors $(a_i,b_i)$, following
$
z_{i,k+1}=\gamma_i(x_k)\,z_{i,k},
$
with $\gamma_i(x)\in[a_i^{\max},1)$ when $x\in P_i$ (observed) and $\gamma_i(x)\in(1,b_i^{\max}]$ otherwise. $z_{max}=10$ for all parts.

\paragraph*{Invariant set computation}
We conservatively instantiate the input set as $U=\{u:\|u\|_{\infty}\leq 0.26\}$ obtained from platform bounds, with an extra inflation $\delta_u=0.1$ to cover Nav2 tracking error. This over-approximates the instantaneous body-frame linear velocity projected in the world frame; heading variation within one sample is absorbed by universal quantification over $u\in U$. For each part $i$ we build a two-mode PWA subsystem in $(x,z_i,u)$ with guards $x\in P_i$ vs.\ $x\notin P_i$ and compute the maximal \ac{rpi} set $S^{\max}_i\subset X\times[0,z_i^{\mathrm{crit}}]\times U$ via fixed-point iteration of the predecessor operator under universal quantification over inputs and LPV vertices (Algorithm~\ref{alg:inv}). Computation took between 111\,s and 154\,s for each compositional invariant on a general-purpose laptop.

\subsection{Monitoring}
From each run we obtain time-stamped robot poses $x_k\in\mathbb{R}^2$ (filtered using a Kalman filter) and per-part uncertainty proxies $z_{i,k}$ (derived from perception/coverage events).
We derive the uncertainty factors $(a_i,b_i)$ online from LiDAR measurements as follows.

\paragraph*{Uncertainty update}
We have excluded the LiDAR detections of the initial mapping pass shown in Figure~\ref{fig:mon:free1} when defining the partitions. To emulate time-varying uncertainty, additional small objects were placed in the labyrinth during operation without preventing the robot from moving freely through all parts. New detections are accumulated per partition in a counter $C_i(k)>0$ at time $k$ (detections since the last time the robot entered $P_i$). When the robot \emph{re-enters} $P_i$ at time $k$, we map $r_i(k) \;:=\; C_i(k)/C_i(k-1)$ monotonically into the decay range $[a_i^{\min},a_i^{\max}]$ by
\[
a_i(k_{\text{in}}) \;=\; \operatorname{clip}\!\Big(a_i^{\min} + \frac{\min\{r_i,\ r_i^{\max}\}-1}{\,r_i^{\max}-1\,}\,\big(a_i^{\max}-a_i^{\min}\big)\Big).
\]
Here $r_i^{\max}>1$ is a chosen saturation point and $\operatorname{clip}(\cdot)$ enforces the interval $[a_i^{\min},a_i^{\max}]$. Intuitively, more detections drive $a_i$ upward toward $a_i^{\max}<1$, meaning uncertainty decays more slowly even when the partition is observed.

For simplicity in this case study, the unobserved growth factor $b_i$ was kept constant, i.e., $b_i\in(1,b_i^{\max}]$ independent of $r_i$. We did not investigate removing obstacles; thus, $C_i^{\mathrm{new}}$ is non-decreasing between entries and $a_i$ is piecewise-constant and non-decreasing across visits. The parameters are updated while moving through $P_i$. The same $z_i$ sequences are then fed to the monitor according to Algorithm~\ref{alg:online}.

\paragraph*{Evaluation}
We monitor a run over multiple minutes in the labyrinth with a failing patrol that loiters and starves two parts, causing uncertainty growth beyond admissible bounds. For reference, we compare against a threshold-only rule that raises an alert when any $z_{i,k}\ge z_i^{\mathrm{crit}}$. Further, we evaluate:
(i) \emph{Detection lead time} from the monitor’s first alert to the threshold crossing $z_{i,k}\ge z_i^{\mathrm{crit}}$ (positive values indicate anticipatory detection); (ii) \emph{false alerts} (alerts without subsequent violation within a grace window); and (iii) checking computation time.

\subsection{Results}
\begin{figure}
  \centering
    \begin{subfigure}{0.48\textwidth}
    \includegraphics[width=\textwidth]{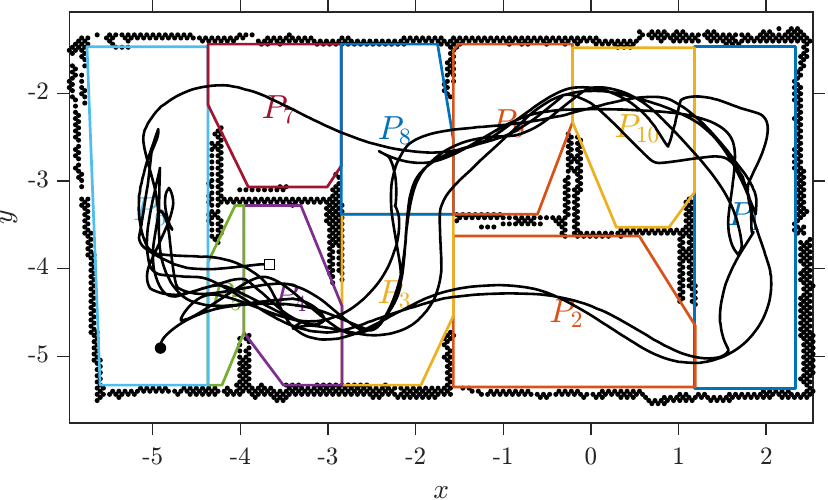}
    \caption{Robot trajectory in labyrinth.}
    \label{fig:mon:traj}
  \end{subfigure}
    \hfill
    \begin{subfigure}{0.48\textwidth}
    \includegraphics[width=\textwidth]{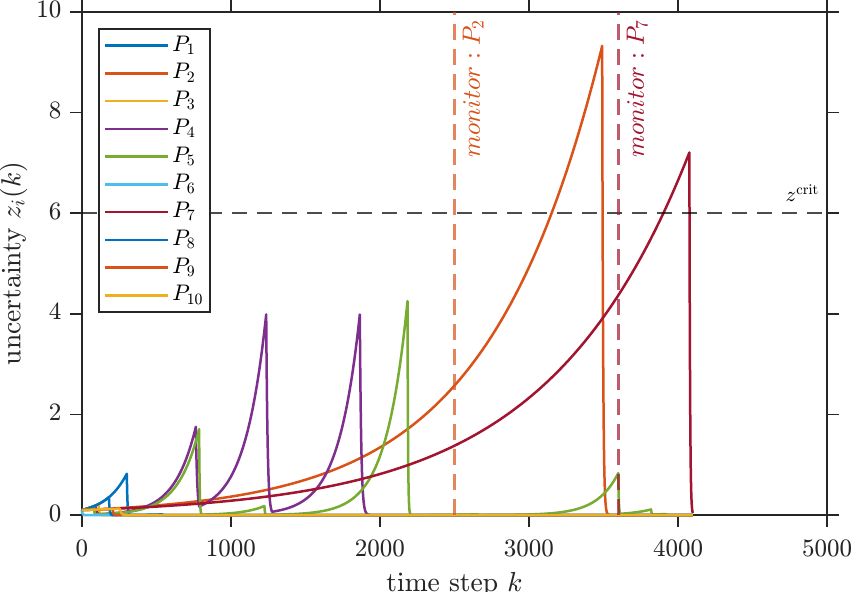}
    \caption{Uncertainty evolution and monitor triggers.}
    \label{fig:mon:uncertain}
  \end{subfigure}
  \vspace{1ex}
  \caption{Labyrinth for persistent surveillance by a four‐wheeled differential‐drive robot. The labyrinth has been manually partitioned into 10 parts.}
  \label{fig:monitor1}
\end{figure}

Figure~\ref{fig:mon:traj} shows the robot position trajectory overlaid on the ten-part partition; Figure~\ref{fig:mon:uncertain} plots the corresponding uncertainty traces and marks monitor triggers. The invariant-based monitor raises alerts \emph{before} uncertainties hit the critical threshold, providing anticipatory warnings; the threshold-only baseline, by definition, reacts at (or after) violation. No false alerts are produced by the monitor.

During the run patrolling fails, loitering in a subset of the maze leads to starvation of two parts; the monitor exits the invariant set well before $z_i$ crosses $z_i^{\mathrm{crit}}$, yielding positive lead time. Subsequent growth of $z_i$ confirms the early warning, and the parts reported by the monitor match the regions visibly under-serviced in the trajectory plot. 

We ran three independent trials per configuration with different random seeds for patrol initialization and obstacle placements. Each of them yielded similar results in terms of monitoring outcomes. On average, per-step runtime for checking polyhedral union membership over all parts was 0.3\,ms; storing per-part unions and their facet matrices dominates memory consumption. We plan to release the scripts for generating the experimental evaluation.

\section{Discussion}
This section reflects on assumptions, practical trade-offs in using the proposed monitor, and directions for extensions.

\paragraph*{Guarantees and assumptions}
The completeness and soundness of the monitor based on compositional \ac{rpi} relies on two conditions: (i) mutual \emph{independence} of the part-wise uncertainties, and (ii) \emph{robust} (universal) quantification over the same state and input space, guards, and LPV bounds. Condition~(i) is appropriate when each uncertainty state is a local proxy driven by sensing staleness/coverage (e.g., time-since-last-observation or map freshness). However, global events may couple the uncertainty of parts. For example, wind-driven smoke propagation in fire monitoring can raise uncertainty beyond one part. In this case, one can cluster coupled parts into multi-part subsystems, albeit at the cost of higher \ac{rpi} complexity. 
In the case study, the under-approximation of the traversable space yields \emph{conservative} invariants. Thus, within the under-approximated area runtime alerts remain \emph{sound} but can be \emph{conservative}.

\paragraph*{Measurement noise and model mismatch}
With noisy $x_k$ or $z_{i,k}$, sharp set boundaries can induce chattering. Two possible standard remedies are (i) \emph{guard erosion}/set inflation: shrink $\mathrm{Obs}_i$ by a margin $\delta$ or inflate the invariant to accommodate estimation error; and (ii) \emph{hysteresis}: declare exit only after $m$ consecutive violations or require a margin $\eta>0$ beyond the boundary. If $u_k$ is not available, quantify over $U$ offline; if it is available, slice the union by the measured $u_k$ to reduce conservatism.

\paragraph*{Complexity}
The offline computation of the \ac{rpi}'s is based on a fixed-point iteration of robust predecessor operators with polyhedral set operations. In the worst case, both runtime and representation size grow exponentially with the state dimension and with the number of halfspaces/vertices. In our compositional construction, this cost is incurred in the dimension of $[x^\top,z_i]^\top$, whereas the full-system \ac{rpi} requires the same operations in $[x^\top,z_1,\ldots,z_N]^\top$. As $N$ increases, computation and storage become quickly prohibitive, which is reflected in Table~\ref{tab:complexity}, where the full invariant has orders-of-magnitude more vertices and higher runtime.

Online runtime scales with the number of polytopes and facets. For compositional monitoring, this scales roughly linearly with $N$ (each membership test is independent), and unions can be compressed via polytope merging, facet pruning, or spatial indexing (e.g., k-d trees on polytope bounding boxes). Upon memory limitations, the union can be approximated by a single outer polytope around each $\mathcal{S}_i^{\max}$ and the approximation gap can be tracked.

\paragraph*{Choice of thresholds}
The thresholds $z_i^{\mathrm{crit}}$ encode the surveillance objective and reflect sensing fidelity, revisit constraints, and acceptable risk. Unequal priorities are handled by part-specific $z_i^{\mathrm{crit}}$ and by weighting the reporting (e.g., alert sooner on critical parts using larger safety margins).

\paragraph*{Observability and motion models}
Both the running example and the case study used polyhedral representations of the environment. For circular cones, occlusions, or complex visibility polygons, inner/outer polyhedral approximations to trade tightness for tractability can be used. The single-integrator kinematics can be replaced by other linear (or modewise linearized) models. The invariant computation proceeds analogously, albeit in a higher-dimensional space and with possibly significantly higher computational complexity.

\paragraph*{Adversarial robustness and cyber–physical attacks}
Because the monitor relies on measured state, integrity attacks that spoof these channels can mask violations. Practical mitigation includes cross-checks between commanded $u_k$ and observed motion (residual tests), redundant sensing for $z_{i,k}$ (e.g., overlapping footprints), and secure logging/attestation for the signals used by the monitor. From a synthesis viewpoint, suspicious channels can be treated as adversarial disturbances by enlarging the LPV bounds, which preserves soundness at the cost of more conservative alerts.

\paragraph*{Limitations}
Exact computation of \ac{rpi} in high dimensions remains challenging; our compositional construction mitigates, but cannot remove, the intrinsic cost for very large $N$ without sacrificing some tightness.

\paragraph*{Extensions}
The framework extends to 3D workspaces, to multi-robot surveillance (modularity remains sound if interactions are conservatively bounded), and to stochastic envelopes (replace worst-case vertices by chance-constrained predecessors to trade false positives for false negatives). Hybridization with other monitoring objectives is another natural extension, e.g., combining with barrier-methods-based monitors for richer early-warning signals, or adding look-ahead for imminent failures.

\section{Conclusion}
We introduced a runtime monitoring framework for persistent surveillance that treats the autonomy stack as a black box. The key idea is a compositional construction of invariant sets for each to-be-surveyed part and checking their conjunction online. Under independence and common-domain assumptions, we showed that the compositional monitor is equivalent to the full-system-based monitor. The monitor requires only constant-time polyhedral membership tests online and scales linearly with the number of parts. A laboratory case study with a mobile robot in a labyrinth illustrated how the construction is instantiated from data and how impending instability can be detected early.

Future work will include incorporating coupling or disturbances in per-part predecessors and extending to multi-robot teams with collision and communication constraints.

\bibliographystyle{IEEEtran}
\bibliography{IEEEabrv,ICRA2026}

@article{smith_persistent_2011,
	title = {Persistent ocean monitoring with underwater gliders: {Adapting} sampling resolution},
	volume = {28},
	number = {5},
	urldate = {2026-02-26},
	journal = {Journal of Field Robotics},
	author = {Smith, Ryan N. and Schwager, Mac and Smith, Stephen L. and Jones, Burton H. and Rus, Daniela and Sukhatme, Gaurav S.},
	year = {2011},
	pages = {714--741},
}

@article{casbeer_cooperative_2006,
	title = {Cooperative forest fire surveillance using a team of small unmanned air vehicles},
	volume = {37},
	number = {6},
	urldate = {2026-02-26},
	journal = {Int. J. Syst. Sci.},
	publisher = {Taylor \& Francis},
	author = {Casbeer, David W. and Kingston, Derek B. and Beard, Randal W. and McLain, Timothy W.},
	year = {2006},
	pages = {351--360},
}

@article{chen_deep_2020,
	title = {Deep {Reinforced} {Learning} {Tree} for {Spatiotemporal} {Monitoring} {With} {Mobile} {Robotic} {Wireless} {Sensor} {Networks}},
	volume = {50},
	number = {11},
	urldate = {2026-02-26},
	journal = {IEEE Trans. Syst., Man, Cybern., Syst.},
	author = {Chen, Jiahong and Shu, Tongxin and Li, Teng and de Silva, Clarence W.},
	year = {2020},
	pages = {4197--4211},
}

@article{ha_periodic_2019,
	title = {On periodic optimal solutions of persistent sensor planning for continuous-time linear systems},
	volume = {99},
	urldate = {2026-02-26},
	journal = {Automatica},
	author = {Ha, Jung-Su and Choi, Han-Lim},
	year = {2019},
	pages = {138--148},
}

@inproceedings{lan_planning_2013,
	title = {Planning periodic persistent monitoring trajectories for sensing robots in {Gaussian} {Random} {Fields}},
	urldate = {2026-02-26},
	booktitle = {IEEE Int. Conf. Robot. Autom. (ICRA)},
	author = {Lan, Xiaodong and Schwager, Mac},
	year = {2013},
	pages = {2415--2420},
}

@inproceedings{hall_persistent_2025,
	title = {Persistent {Monitoring} {Trajectory} {Optimization} in {Partitioned} {Environments}},
	urldate = {2026-02-26},
	booktitle = {{American} {Control} {Conf.} ({ACC})},
	author = {Hall, Jonas and Cassandras, Christos G. and Andersson, Sean B.},
	year = {2025},
	pages = {3813--3818},
}

@article{blanchini_set_1999,
	title = {Set invariance in control},
	volume = {35},
	number = {11},
	urldate = {2026-02-26},
	journal = {Automatica},
	author = {Blanchini, F.},
	year = {1999},
	pages = {1747--1767},
}

@article{yu_persistent_2015,
	title = {Persistent {Monitoring} of {Events} {With} {Stochastic} {Arrivals} at {Multiple} {Stations}},
	volume = {31},
	number = {3},
	urldate = {2026-02-26},
	journal = {IEEE Transactions on Robotics},
	author = {Yu, Jingjin and Karaman, Sertac and Rus, Daniela},
	year = {2015},
	pages = {521--535},
}

@inproceedings{jones_information-guided_2015,
	title = {Information-guided persistent monitoring under temporal logic constraints},
	urldate = {2026-02-26},
	booktitle = {2015 {American} {Control} {Conference} ({ACC})},
	author = {Jones, Austin and Schwager, Mac and Belta, Calin},
	year = {2015},
	note = {ISSN: 2378-5861},
	pages = {1911--1916},
}

@article{zhou_optimal_2018,
	title = {Optimal {Event}-{Driven} {Multiagent} {Persistent} {Monitoring} of a {Finite} {Set} of {Data} {Sources}},
	volume = {63},
	number = {12},
	urldate = {2026-02-26},
	journal = {IEEE Trans. Autom. Control},
	author = {Zhou, Nan and Yu, Xi and Andersson, Sean B. and Cassandras, Christos G.},
	year = {2018},
	pages = {4204--4217},
}

@inproceedings{smith_persistent_2011-1,
	title = {Persistent monitoring of changing environments using a robot with limited range sensing},
	urldate = {2026-02-26},
	booktitle = {IEEE Int. Conf. Robot. Autom. (ICRA)},
	author = {Smith, Stephen L. and Schwager, Mac and Rus, Daniela},
	year = {2011},
	pages = {5448--5455},
}

@article{leucker_brief_2009,
	series = {The 1st {Workshop} on {Formal} {Languages} and {Analysis} of {Contract}-{Oriented} {Software} ({FLACOS}’07)},
	title = {A brief account of runtime verification},
	volume = {78},
	number = {5},
	urldate = {2026-02-26},
	journal = {The Journal of Logic and Algebraic Programming},
	author = {Leucker, Martin and Schallhart, Christian},
	year = {2009},
	pages = {293--303},
}

@inproceedings{maler_monitoring_2004,
	address = {Berlin, Heidelberg},
	title = {Monitoring {Temporal} {Properties} of {Continuous} {Signals}},
	language = {en},
	booktitle = {Formal {Techniques}, {Modelling} and {Analysis} of {Timed} and {Fault}-{Tolerant} {Systems}},
	publisher = {Springer},
	author = {Maler, Oded and Nickovic, Dejan},
	editor = {Lakhnech, Yassine and Yovine, Sergio},
	year = {2004},
	pages = {152--166},
}

@inproceedings{nenchev_monitoring_2026,
	address = {Cham},
	title = {Monitoring {Progress} and {Failure} in {Autonomous} {Robot} {Navigation}: {A} {Case} {Study}},
	language = {en},
	booktitle = {Runtime {Verification}},
	publisher = {Springer Nature Switzerland},
	author = {Nenchev, Vladislav and Sotiriadis, Prodromos},
	editor = {Könighofer, Bettina and Torfah, Hazem},
	year = {2026},
	pages = {317--335}
}

@article{mitchell_time-dependent_2005,
	title = {A time-dependent {Hamilton}-{Jacobi} formulation of reachable sets for continuous dynamic games},
	volume = {50},
	number = {7},
	urldate = {2026-02-26},
	journal = {IEEE Trans. Autom. Control},
	author = {Mitchell, I.M. and Bayen, A.M. and Tomlin, C.J.},
	year = {2005},
	pages = {947--957},
}

@article{ames_control_2017,
	title = {Control {Barrier} {Function} {Based} {Quadratic} {Programs} for {Safety} {Critical} {Systems}},
	volume = {62},
	number = {8},
	urldate = {2026-02-26},
	journal = {IEEE Trans. Autom. Control},
	author = {Ames, Aaron D. and Xu, Xiangru and Grizzle, Jessy W. and Tabuada, Paulo},
	year = {2017},
	pages = {3861--3876},
}

@article{mayne_model_2014,
	title = {Model predictive control: {Recent} developments and future promise},
	volume = {50},
	number = {12},
	urldate = {2026-02-26},
	journal = {Automatica},
	author = {Mayne, David Q.},
	year = {2014},
	pages = {2967--2986},
}

@incollection{shamma_overview_2012,
	address = {Boston, MA},
	title = {An {Overview} of {LPV} {Systems}},
	booktitle = {Control of {Linear} {Parameter} {Varying} {Systems} with {Applications}},
	publisher = {Springer US},
	author = {Shamma, Jeff S.},
	editor = {Mohammadpour, Javad and Scherer, Carsten W.},
	year = {2012},
	pages = {3--26},
}

@article{fainekos_robustness_2009,
	title = {Robustness of temporal logic specifications for continuous-time signals},
	volume = {410},
	number = {42},
	urldate = {2026-02-26},
	journal = {Theoretical Computer Science},
	author = {Fainekos, Georgios E. and Pappas, George J.},
	year = {2009},
	pages = {4262--4291},
}

@inproceedings{nenchev_receding_2016,
	title = {Receding {Horizon} {Robot} {Control} in {Uncertain} {Environments} with {Temporal} {Logic} {Constraints}},
	booktitle = {15th {European} {Control} {Conf}. ({ECC})},
	author = {Nenchev, V. and Belta, C.},
	year = {2016},
	pages = {2614--2619},
}

@article{nenchev_event-driven_2018,
	title = {Event-driven optimal control for a robotic exploration, pick-up and delivery problem},
	volume = {30},
	journal = {Nonlinear Analysis: Hybrid Systems},
	author = {Nenchev, V. and Cassandras, C. G. and Raisch, J.},
	year = {2018},
	pages = {266--284},
}

\end{document}